\documentclass{article}

\usepackage{microtype}
\usepackage{graphicx}
\usepackage{subfigure}
\usepackage{booktabs} 
\usepackage{mathmacros}
\usepackage{amsmath}
\usepackage{amssymb}
\usepackage{amsthm}

\usepackage{hyperref}
\usepackage{xcolor}
\usepackage[T1]{fontenc}

\usepackage[accepted]{icml2020}

\newcommand{\bbet}{\sm{\beta}}
\newcommand{\bnorm}{\beta}
\newcommand{\x}{\m{x}}
\newcommand{\y}{\m{y}}
\newcommand{\F}{\m{F}}
\renewcommand{\H}{\m{H}^{\infty}}

\newcommand{\X}{\m{X}}
\renewcommand{\a}{\alpha}

\newcommand{\eps}{\epsilon}
\newcommand{\rrat}{R}
\newcommand{\h}{\m{h}}
\newcommand{\w}{\m{w}}
\newcommand{\thet}{\sm{\theta}}
\newcommand{\Lo}{\mathcal{L}}

\newtheorem*{arorathm}{Theorem (\cite{arora_finegrained_2019}, 3.3)}
\newtheorem*{aroracor}{Corollary (\cite{arora_finegrained_2019}, 6.2)}
\newtheorem*{aroracorext}{Extension of Corollary (\cite{arora_finegrained_2019}, 6.2)}

\newtheorem{theorem}{Theorem}
\newtheorem{subtheorem}{Theorem}[theorem]
\newtheorem{corollary}{Corollary}[theorem]
\newtheorem{lemma}{Lemma}
\newtheorem{definition}{Definition}

\makeatletter
\newtheorem*{rep@theorem}{\rep@title}
\newcommand{\newreptheorem}[2]{%
\newenvironment{rep#1}[1]{%
 \def\rep@title{#2 \ref{##1}}%
 \begin{rep@theorem}}%
 {\end{rep@theorem}}}
\makeatother

\newreptheorem{theorem}{Theorem}
\newreptheorem{subtheorem}{Theorem}
\newreptheorem{corollary}{Corollary}
\newreptheorem{lemma}{Lemma}

\begin{document}

\twocolumn[
\icmltitle{Learning the gravitational force law and other analytic functions}



\icmlsetsymbol{equal}{*}

\begin{icmlauthorlist}
\icmlauthor{Atish Agarwala}{equal,goog}
\icmlauthor{Abhimanyu Das}{goog}
\icmlauthor{Rina Panigrahy}{goog}
\icmlauthor{Qiuyi Zhang}{goog}
\end{icmlauthorlist}

\icmlaffiliation{goog}{Google Research}

\icmlcorrespondingauthor{Atish Agarwala}{thetish@google.com}

\icmlkeywords{Machine Learning Theory, Kernel learning, ICML}

\vskip 0.3in
]



\printAffiliationsAndNotice{}  

\begin{abstract}
Large neural network models have been successful in learning functions
of importance in many branches of science, including physics, chemistry
and biology. Recent theoretical work has shown explicit learning bounds
for wide networks and kernel methods on some simple classes of functions,
but not on more complex functions which arise in practice. We extend these
techniques to provide learning bounds for analytic functions on the sphere
for any kernel method or equivalent infinitely-wide network with the corresponding
activation function trained with SGD. We show that a wide, one-hidden layer ReLU
network can learn analytic functions with a number of samples proportional
to the derivative of a related function. Many functions important
in the sciences are therefore efficiently learnable. As an example,
we prove explicit bounds on learning the many-body gravitational force
function given by Newton's law of gravitation. Our theoretical bounds suggest
that very wide ReLU networks (and the corresponding NTK kernel) are better
at learning analytic functions as compared to kernel learning with Gaussian kernels.
We present experimental evidence that the many-body gravitational force function
is easier to learn with ReLU networks as compared to networks with exponential activations.
\end{abstract}

\section{Introduction}

\label{sec:intro}

Despite the empirical success of neural networks and other
highly parameterized machine learning methods, a major open question
remains: why do these methods perform well? Classical
learning theory does not predict or explain the success of
large, often overparameterized models \cite{lawrence_lessons_1997,harvey_nearlytight_2017,bartlett_nearlytight_2019}. Most models
are highly expressive 
\cite{poole_exponential_2016,raghu_expressive_2017}, but can still generalize when 
trained
with many less data points than parameters 
\cite{he_deep_2016,szegedy_rethinking_2016,dai_transformerxl_2019}.

In particular, many functions of importance to physics 
\cite{mills_deep_2017,gao_efficient_2017,carrasquilla_machine_2017,raissi_hidden_2018},
chemistry \cite{rupp_fast_2012,faber_prediction_2017}, and biology 
\cite{kosciolek_accurate_2016,ainscough_deep_2018,xu_distancebased_2019}, among other fields,
can be learned using neural networks and kernel methods.
This leads to the question:
can we understand what types of functions can be
learned efficiently with particular methods?
Recent theoretical work has focused on answering this question by
constructing bounds on the generalization error, given properties
of the model and data distribution.
Specifically, for very wide networks with specific kernels or activations 
\cite{jacot_neural_2018,du_gradient_2019,allen-zhu_learning_2019,arora_finegrained_2019,lee_wide_2019}, data-dependent 
generalization bounds can be derived by relating wide networks to kernel learning with a specific network-induced kernel, known as 
the  \emph{neural tangent kernel} (NTK)
\cite{jacot_neural_2018}. These bounds while not tight, can be used to mathematically justify why neural networks
trained on noisy labels can achieve low training error but will fail to generalize, while those trained on real data
generalize well. It is unclear, however, whether they
give a sense of the relative difficulty of learning different types
of functions with different types of methods.

\subsection{Our contributions} 

To explain and further understand the efficacy of deep learning in numerous applications, we present generalization bounds on learning analytic functions
on the unit sphere
with any kernel method or sufficiently wide neural network
(Section \ref{sec:kernel_learning}).
In the particular case of a wide, one-hidden layer ReLU network, we present a succinct bound on the number of samples needed to 
guarantee low test error. Informally, we prove the following:
\begin{repsubtheorem}{thm:univar}[informal]
Given an analytic function $g(y)$,
the function $g(\bbet\cdot\x)$, for fixed
$\bbet\in\mathbb{R}^{d}$ and inputs $\x\in\mathbb{R}^{d}$ is
learnable to error $\eps$ with
$O((\bnorm \tilde{g}'(\bnorm)+\tilde{g}(0))^2/\eps^2)$ samples,
with
\begin{equation}
\tilde{g}(y) = \sum_{k=0}^{\infty}|a_k|x^k
\end{equation}
where the $a_k$ are the power series coefficients of $g(y)$.
\end{repsubtheorem}

We prove a much more general version
for multivariate analytic functions:
\begin{repsubtheorem}{thm:multivar}[informal]
Given a multivariate analytic function $g(\x)$ for $\x$ in the
$d$-dimensional unit ball, there is a function
$\tilde{g}(y)$ as defined in Theoreom~\ref{thm:univar} such that $g(\x)$ is learnable to error $\eps$
with $O(\tilde{g}'(1)/\eps^2)$ samples.
\end{repsubtheorem}

Using Theorem \ref{thm:multivar}, we develop a \emph{calculus of bounds} - 
showing that the sum, product, and composition of learnable functions is also learnable,
with bounds constructed using the familiar product and chain rules of
univariate calculus (Corollaries \ref{cor:prod_rule}, 
\ref{cor:chain_rule}, and \ref{cor:multivar_anal}).
These bounds also can be applied when
$g(\x)$ has a singularity, provided that the data is sampled
away from the singularity.

Since many functions used in scientific theories and models fall into this function class, our calculus allows for a clear quantifiable explanation for why neural network models have had successful applications to
many of those fields.
As an important example
from physics, we consider
the forces $\{\F_{i}\}$ between $k$ bodies with positions $\x_{i}$ interacting via
Newtonian gravitation:
\begin{equation}
\label{eq:gravity}
\F_{i} = \sum_{j\neq i} \frac{m_{i}m_{j}}{||\x_{i}-\x_{j}||^3}(\x_{j}-\x_{i})
\end{equation}
We show that,
as long as there is some minimum distance between the
$\x_{i}$ and $\x_{j}$, we can still use the calculus of bounds to show that the $k$
force vectors can be efficiently learned. We prove the following:
\begin{reptheorem}{thm:gravity_bound}[informal]
A wide, one-hidden
layer ReLU network
can learn the force law between $k$ gravitational bodies up to
error $\eps$ using only
$k^{O(\ln(k^2/\eps))}$ samples.
\end{reptheorem}

Lastly, we compare our generalization bounds for the ReLU network with those for more traditional kernel learning methods. 
Specifically, we show asymptotically
weaker bounds for other models, including for
kernel regression with Gaussian kernels, providing some theoretical evidence why neural networks with ReLU activation (or
their
induced kernels) often achieve superior performance than the standard Gaussian kernels. We support our theoretical work with 
numerical experiments on a synthetic dataset that show that wide networks can learn the gravitational force law with
minimal fine-tuning, and achieve lower test and generalization error than the standard Gaussian kernel counterpart. Our 
results suggest that networks with better theoretical learning bounds may in fact
perform better in practice as well, even when theoretical bounds are pessimistic.

\section{Prelimaries}
\subsection{Background}
\label{sec:prev_work}

Classical tools like VC dimension \cite{vapnik_nature_2000} are insufficient to explain the performance of
overparameterized neural networks \cite{bartlett_nearlytight_2019}. Networks generalize well even though
are often very expressive
\cite{poole_exponential_2016,raghu_expressive_2017}, can memorize random data \cite{zhang_understanding_2017,arpit_closer_2017},
and have correspondingly large VC
dimensions \cite{maass_neural_1994,harvey_nearlytight_2017,bartlett_nearlytight_2019}.

Some progress has been made by focusing on the ``implicit regularization''
provided by training dynamics \cite{gunasekar_characterizing_2018,du_power_2018,arora_stronger_2018}.
In particular, SGD biases networks to
solutions with small weight changes under the $\ell_2$ norm (plus any additional
regularization), which has been used to inspire various norm-based bounding
strategies \cite{neyshabur_role_2018,allen-zhu_learning_2019,arora_finegrained_2019}. While many of
these bounds must be computed post-training, some bounds can be computed
using the architecture alone, and show that function classes
like outputs of small networks with smooth activations can be
efficiently learned with large networks \cite{allen-zhu_learning_2019,arora_finegrained_2019, du_gradient_2019}.

Recently, learning bounds for very wide networks have been derived by
combining insights on learning dynamics with more classical generalization error 
bounds in kernel learning. In the limit of infinite width, the total
change in each individual parameter is small, and the outputs of
the network are linear in the weight changes and it can be shown that the learning dynamics
are largely governed by the \emph{neural tangent kernel}
(NTK) of the corresponding network
\cite{jacot_neural_2018}. Given the $n\times 1$ dimensional
vector $\hat{\m{y}}(\thet)$ of model outputs on the training data,
as a function of trainable parameters $\thet$, the
\emph{empirical} tangent kernel is the
$n\times n$ matrix
given by
\begin{equation}
\m{H} = 
\frac{\partial \y}{\partial \thet}\left(\frac{\partial \y}{\partial \thet}\right)^{\tpose}
\end{equation}
Since the derivatives are only taken with respect to the parameters
that are trained, the tangent kernel is different if
particular layers of the network are fixed after initialization.
The empirical kernel
concentrates around some limiting matrix $\H$ in the limit
of a large number of parameters. The NTK kernel evaluated on
two inputs $\x, \x'$ corresponds to the limiting value of
$\H$ evaluated at $\x$ and $\x'$.

As an example, the kernel function
$K(\x,\x')$ for inputs $\x$ and $\x'$
into a single hidden layer fully-connected network,
with non-linearity $\phi$,
with {\it only final layer weights trained} is
\cite{jacot_neural_2018,lee_wide_2019}
\begin{equation}
K(\x,\x') = \expect[\phi(\x\cdot \m{z})\phi(\x'\cdot\m{z})]
\label{eq:ntk_kernel}
\end{equation}
where the $\m{z}_i$ are i.i.d. Gaussian random variables with the same variance
$\sigma^2$ as the hidden layer weights. Note here that
$H$ is a function of $\x\cdot\x$, $\x\cdot\x'$ and $\x'\cdot\x'$ only.

With MSE loss, the learning dynamics
in the wide network regime
is similar to
kernel regression or kernel learning.
Rademacher complexity \cite{koltchinskii_rademacher_2000} can be used
to generate learning bounds for wide networks near the infinite width limit, such as
in \cite{arora_finegrained_2019}. The following theorem shows
that if for training labels $\y$, the product $\y^{\tpose}(\H)^{-1}\y$ is bounded
by $M_{g}$, then wide networks trained with SGD have error less than $\eps$ when
trained with $O(M_g/\eps^2)$ samples:

\begin{arorathm}
\label{thm:arora}
Let $g(\x)$ be a function over $\mathbb{R}^{d}$, and
$\mathcal{D}$ be a distribution over the inputs. Let
$\Lo$ be a 1-Lipschitz loss function. Consider training
a two-layer ReLu network to learn $g$ using SGD with MSE loss on $n$
i.i.d. samples from $\mathcal{D}$. Define the generalization
error $E_{gen}$ of the trained model $\hat{g}(\x)$ as
\begin{equation}
E_{gen} = \expect_{\x\sim\mathcal{D}}[\Lo(g(\x),\hat{g}(\x))]
\label{eq:gen_error_def}
\end{equation}

Fix a failure probability $\delta$.
Suppose that with probability greater than $\delta/3$, $\H$ has smallest eigenvalue $\lambda_0>0$.
Then for a wide enough network,
with probability at least $1-\delta$, the generalization error
is at most
\begin{equation}
E_{gen} \leq \sqrt{\frac{2\y^{\tpose}(\H)^{-1}\y}{n}}+O\left(\sqrt{\frac{\log(n/\lambda_0\delta)}{n}}\right)
\label{eq:gen_error_eq}
\end{equation}
where $\H$ is the $n\times n$ Gram matrix whose elements correspond
to the NTK kernel evaluated at pairs of the $n$ i.i.d. training examples,
and $\y$ is the $n$-dimensional vector of training labels.

If there is some $n$-independent constant $M_{g}$ such that
$\y^{\tpose}(\H)^{-1}\y\leq M_{g}$, then with probability
at least $1-\delta$,
$O([M_{g}+\log(\delta^{-1})]/\eps^2)$ samples are sufficient
to ensure generalization error less than $\eps$.
\end{arorathm}

\subsection{Notation}

We define $\| \cdot \|$ to be the Euclidean norm, unless otherwise specified and $\x \cdot \x'$ to be the dot product between vectors $\x, \x'$. For a vector $\x$ and a scalar $c$, we define $\x + c = \x + c{\bf 1}$ and other operations analogously. 
For the remainder of this paper, we focus on the learnability of functions under different learning algorithms and so we will define \emph{efficiently learnable}
functions as:
\begin{definition}
\label{def:eff_learnable}
Given a learning algorithm, we say that a function $g$ over a distribution of inputs $\mathcal{D}$
is \emph{efficiently learnable} if, given an error scale $\eps$, with probability greater than $1-\delta$,
the generalization error $\expect_{\x\sim\mathcal{D}}[\Lo(g(\x),\hat{g}(\x))]$ of the trained model
$\hat{g}$ with respect to any 1-Lipschitz loss function $\Lo$
is less than $\eps$ when the training data consists of at least $O([M_g+\log(\delta^{-1})]/\eps^2)$
i.i.d. samples drawn from $\mathcal{D}$, for some $n$-independent constant $M_g$.
\end{definition}

\section{Theory}

\label{sec:theory}

\subsection{Kernel learning bounds}

\label{sec:kernel_learning}

In this section, we extend the bounds derived in \cite{arora_finegrained_2019} 
to any kernel that
can be written as a power series in
the dot product of inputs $\x\cdot\x'$. We emphasize that our kernel learning bounds can be generalized to the setting where we train a wide neural network on our data. In Appendix \ref{sec:kernel_two_layer},
we make this relation rigorously clear and show that
Equation \ref{eq:gen_error_eq} applies when training the upper layer only
of any wide network - which is equivalent to a draw from
the posterior of a Gaussian process with the NTK kernel given by Equation \ref{eq:ntk_kernel}. Therefore, we focus on kernels in this section.

We can extend the following corollary, originally
proved for wide ReLU networks with trainable hidden layer only:

\begin{aroracor}
\label{cor:arora}
Consider the function $g:\mathbb{R}^{d}\to\mathbb{R}$ given by:
\begin{equation}
g(\m{x}) = \sum_{k} a_{k}(\bbet_{k}^{\tpose}\x)^{k}
\end{equation}
Then, if $g$ is restricted to $||\m{x}||=1$, and
the NTK kernel can be written as $H(\x,\x') = \sum_{k}b_k(\x\cdot\x')^k$,
the function can be learned efficiently with
a wide one-hidden-layer network in the sense of Definition \ref{def:eff_learnable}
with
\begin{equation}
\sqrt{M_{g}} = \sum_{k} b_k^{-1/2} |a_{k}| ||\bbet_{k}||_{2}^{k}
\label{eq:arora_bound}
\end{equation}
up to $g$-independent constants of $O(1)$, where $\bnorm_k\equiv ||\bbet_{k}||_{2}$.
In the particular case of a ReLU network, the bound is
\begin{equation}
\sqrt{M_{g}} = \sum_{k} k |a_{k}| ||\bbet_{k}||_{2}^{k}
\label{eq:arora_bound_relu}
\end{equation}
if the $a_k$ are non-zero only for
$k = 1$ or $k$ even.
\end{aroracor}

Using Equation \ref{eq:gen_error_eq}
and the arguments of Appendix \ref{sec:kernel_two_layer}, we arrive at the following
extension:
\begin{aroracorext}
Consider a kernel method or appropriately wide network with only the upper layer trained,
with kernel
\begin{equation}
K(\x,\x') = \sum_{k} b_{k}(\x\cdot\x')^k
\label{eq:kernel_series}
\end{equation}
Then the learning bound in Equation \ref{eq:arora_bound} holds for these
models as well.
\end{aroracorext}

Building off of this learning bound, we will prove in Section~\ref{sec:learning_analytic} that all analytic functions are 
efficiently learnable, via both kernel methods and
wide networks.

Equation \ref{eq:arora_bound} suggests that kernels with slowly decaying (but still convergent)
$b_{k}$ will give the best bounds for learning polynomials.
Many popular kernels do not meet this criteria. For example, for inputs
on the sphere of radius $r$, the Gaussian kernel $K(\x,\x') = e^{-||\x-\x'||^2/2}$
can be written as
$K(\x,\x') = e^{-r^2}e^{\x\cdot\x'}$. This has $b_k^{-1/2} = e^{r^2/2}\sqrt{k!}$, which increases
rapidly with $k$. This provides theoretical
justification for the empirically inferior performance of the Gaussian kernel which we will
present in Section \ref{sec:expts}.

Guided by this theory, we focus on
kernels where
$b_{k}^{-1/2} \leq O(k)$, for all $k$.
While the ReLU NTK kernel (with inputs on the sphere) satisfies this 
bound for even positive powers $k$, it fails to satisfy our criteria 
for odd values of $k$. One way to ensure the
bound exists for all $k > 0$ is to construct a kernel by hand:
for example,
\begin{equation}
K(\x,\x') = \sum_{k} k^{-s}(\x\cdot\x')^k
\end{equation}
with $s\in(1,2]$ is a valid slowly decaying kernel on the sphere.

Another approach, which keeps the model similar to those used in 
practice,
is to introduce a novel kernel by applying the following
modification to
the NTK kernel. Consider appending a constant
component to the input $\x$
so that the new input to the network is $(\x/\sqrt{2},1/\sqrt{2})$.
The kernel then becomes:
\begin{equation}
\label{eq:ntk-bias}
K(\x,\x') = \frac{\x\cdot\x'+1}{4\pi}\left(\pi-\arccos\left(\frac{\x\cdot\x'+1}{2}\right)\right)
\end{equation}
Re-writing the power series as an expansion around
$\x\cdot\x' = 0$, we have terms of all powers. An asymptotic analysis
of the coefficients (Appendix \ref{sec:mod_ReLU})
shows that coefficients $b_k$ are asymptotically
$O(k^{-3/2})$ - meeting our needs. In particular, this means that the bound
in Equation \ref{eq:arora_bound_relu} applies to these
kernels, without restriction to even $k$. Note that
for $k=0$,
the constant function $g(\m{x}) = a_{0}$ can be learned with
$\sqrt{M_{g}} = |a_{0}|$ samples.

\subsection{Learning analytic functions}

\label{sec:learning_analytic}

For the remainder of this section,
we assume that we are using a GP/wide network
with a kernel $K$ of the form
$K(\x,\x') = \sum_{k=0}^{\infty} b_k (\x\cdot\x')^k$.
Unless otherwise noted, we also assume that
$b_k \geq k^{-2}$ for large $k$
so Equation \ref{eq:arora_bound_relu} applies for
all powers of $k$.
We will use this to show that all univariate \emph{analytic functions} are efficiently learnable,
and then extend the results to multivariate functions.

Analytic functions are a rich class with a long history of use in the
sciences and applied mathematics. Functions are analytic if they
have bounded derivatives of all orders when extended to the complex
plane. This is equivalent to having a locally convergent power
series representation, a fact which we will exploit for many of our
proofs.

\subsubsection{Univariate analytic functions}

We start with the univariate case and first prove the following:
\begin{subtheorem}
\label{thm:univar}
Let $g(y)$ be a function analytic around $0$, with radius of convergence
$R_{g}$.
Define the \emph{auxiliary function} $\tilde{g}(y)$ by the power series
\begin{equation}
\tilde{g}(y) = \sum_{k=0}^{\infty} |a_{k}| y^k
\end{equation}
where the $a_k$ are the power series coefficients of $g(y)$. Then the function
$g(\bbet\cdot\x)$, for some fixed vector $\bbet\in\mathbb{R}^{d}$ with $||\x|| = 1$ is efficiently
learnable in the sense of Definition \ref{def:eff_learnable}
using a model with the slowly decaying kernel $K$ with
\begin{equation}
\sqrt{M_{g}} = \bnorm \tilde{g}'(\bnorm)+\tilde{g}(0)
\end{equation}
if the norm $\bnorm\equiv||\bbet||_{2}$ is less than $R_{g}$.
\end{subtheorem}

\begin{proof}
We first note that the radius of convergence of the power series of $\tilde{g}(y)$ is also
$R_{g}$ since $g(y)$ is analytic. Applying Equation \ref{eq:arora_bound_relu}, pulling out
the $0$th order term, and factoring out $\bnorm$,
we get
\begin{equation}
\sqrt{M_{g}} = |a_{0}|+\bnorm\sum_{k=1}^{\infty}k|a_{k}|\bnorm^{k} = \bnorm\tilde{g}'(\bnorm)+\tilde{g}(0)
\end{equation}
since $\bnorm<R_{g}$.
\end{proof}

The relationship between $\tilde{g}(y)$ and the original $g(y)$ depends on the power series representation
of $g(y)$. For example, for $g(y) = 1/(1-y)$, the power series has all positive coefficients and $g(y) = \tilde{g}(y)$.
The worst case scenario is when the power series has alternating sign; for example, for $g(y) = e^{-y^2}$,
$\tilde{g}(y) = e^{y^2}$.

\subsubsection{Multivariate analytic functions}

\label{sec:multi_var}

The above class of efficiently learnable
functions is somewhat limiting; it is
``closed'' over addition (the sum of learnable functions is
learnable), but not over products and composition. The following lemma,
proved in Appendix \ref{sec:multivar_proof}, allows us to
generalize:

\begin{lemma}
\label{lem:multivar}
Given a collection of $p$ vectors $\bbet_{i}$ in $\mathbb{R}^d$,
the function $f(\x) = \prod_{i=1}^{p} \bbet_{i}\cdot \x$ is 
efficiently learnable with
\begin{equation}
\sqrt{M_{f}} = p\prod_{i}\bnorm_{i}
\end{equation}
where $\bnorm_{i}\equiv ||\bbet_{i}||_{2}$.
\end{lemma}

Using this lemma we can prove:

\begin{subtheorem}
\label{thm:multivar}

Let $g(\x)$ be a function with multivariate power series representation:
\begin{equation}
g(\x) = \sum_{k} \sum_{v\in V_k} a_{v} \prod_{i=1}^{k} (\bbet_{v,i}\cdot\x)
\end{equation}
where the elements of $V_k$ index the $k$th order
terms of the power series. We define $\tilde{g}(y) = \sum_{k} \tilde{a}_{k} y^k$
with coefficients 
\begin{equation}
\tilde{a}_{k} = \sum_{v\in V_{k}} |a_{v}|\prod_{i=1}^{k}\bnorm_{v,i}
\end{equation}

If the power series of $\tilde{g}(y)$ converges at $y=1$ then with high probability
$g(\x)$ can be learned efficiently in the sense of Definition \ref{def:eff_learnable}
with $\sqrt{M_{g}} = \tilde{g}'(1)+\tilde{g}(0)$.
\end{subtheorem}

\begin{proof}
Follow the construction in Theorem \ref{thm:univar}, using Lemma \ref{lem:multivar} to get bounds on the individual terms.
Then sum and evaluate the power series of $\tilde{g}'(1)$ to arrive at the bound.
\end{proof}

Since the set of efficiently
learnable functions is now appropriately ``closed''
over addition and multiplication, the standard machinery of calculus
can be used to prove learning bounds for combinations of functions
with known bounds. For example, we have:

\begin{corollary}[Product rule]
\label{cor:prod_rule}
Let $g(\x)$ and $h(\x)$ meet the conditions of Theorem 1. Then the product
$g(\x)h(\x)$ is efficiently learnable as well, with bound
\begin{equation}
\sqrt{M_{gh}} = \tilde{g}'(1)\tilde{h}(1)+\tilde{g}(1)\tilde{h}'(1)+\tilde{g}(0)\tilde{h}(0)
\end{equation}
\end{corollary}

\begin{proof}
Consider the power series of $g(\x)h(\x)$, which exists and is convergent since each individual
series exists and is convergent. Let the elements of $V_{j,g}$ and $V_{k,h}$ index the $j$th order terms of $g$ and the $k$th order
terms of $h$ respectively.
The individual terms in the series look like:
\begin{equation}
a_{v}b_{w} \prod_{j'=1}^{j} (\bbet_{v,j'}\cdot\x)\prod_{k'=1}^{k} (\bbet_{w,k'}\cdot\x)~\text{for}~v\in V_{j,g},~w\in V_{k,h}
\end{equation}
with bound
\begin{equation}
(j+k)|a_{v}||b_{w}| \prod_{j'=1}^{j} \bnorm_{v,j'}\prod_{k'=1}^{k} \bnorm_{w,k'}~\text{for}~v\in V_{j,g},~w\in V_{k,h}
\end{equation}
for all terms with $j+k >0$ and $\tilde{g}(0)\tilde{h}(0)$ for the term with $j= k = 0$.

Distribute the $j+k$ product, and first focus on the $j$ term only. Summing over all the $V_{k,h}$ for all $k$, we get
\begin{equation}
\begin{split}
\sum_{k} \sum_{w\in V_{k,h}}j |a_{v}||b_{w}| \prod_{j'=1}^{j} \bnorm_{v,j'}\prod_{k'=1}^{k} \bnorm_{w,k'}& = \\ |a_{v}|\prod_{j'=1}^{j} \bnorm_{v,j'}\tilde{h}(1)
\end{split}
\end{equation}
Now summing over the $j$ and $V_{j,g}$ we get $\tilde{g}'(1)\tilde{h}(1)$. If we do the same for the $k$ term,
after summing we get $\tilde{g}(1)\tilde{h}'(1)$. These bounds add and we get the desired formula for
$\sqrt{M_{gh}}$, which, up to the additional $\tilde{g}(0)\tilde{h}(0)$ term looks is the product rule applied to $\tilde{g}$ and $\tilde{h}$.
\end{proof}

One immediate application for this corollary is the product of many
univariate analytic functions. If we define
\begin{equation}
G(\x) = \prod_{i} g_{i}(\bbet_{i}\cdot\x)
\end{equation}
where each of the corresponding $\tilde{g}_{i}(y)$ have the appropriate convergence properties,
then $G$ is efficiently learnable with bound $M_{G}$ given by
\begin{equation}
\sqrt{M_{G}} = \left.\frac{d}{dy} \prod_{i} \tilde{g}_{i}(\bnorm_{i} y)\right|_{y=1}+\prod_{i} \tilde{g}_{i}(0)
\end{equation}

We can also derive the equivalent of the chain rule for
function composition:

\begin{corollary}[Chain rule]
\label{cor:chain_rule}
Let $g(y)$ be an analytic function
and $h(\x)$ be efficiently learnable, with auxiliary functions $\tilde{g}(y)$ and
$\tilde{h}(y)$ respectively. Then the composition $g(h(\x))$ is efficiently learnable as well with bound
\begin{equation}
\sqrt{M_{g\circ h}} = \tilde{g}'(\tilde{h}(1))\tilde{h}'(1)+\tilde{g}(\tilde{h}(0))
\end{equation}
provided that $\tilde{g}(\tilde{h}(0))$ and $\tilde{g}(\tilde{h}(1))$ converge
(equivalently, if $\tilde{h}(0)$ and $\tilde{h}(1)$ are in the radius of convergence of $g$).
\end{corollary}

\begin{proof}
Writing out $g(h(\x))$ as a power series in $h(\x)$, we have:
\begin{equation}
g(h(\x)) = \sum_{k=0}^{\infty} a_{k} (h(\x))^{k}
\end{equation}
We can bound each term individually, and use the $k$-wise product rule to bound each term of
$(h(\x))^{k}$. Doing this, we have:
\begin{equation}
\sqrt{M_{g\circ h}} = \sum_{k=1}^{\infty} k|a_{k}|\tilde{h}'(1)\tilde{h}(1)^{k-1}+\sum_{k=0}^{\infty}|a_{k}|\tilde{h}(0)^{k}
\end{equation}
Factoring out $\tilde{h}'(1)$ from the first term and then evaluating each of the series gets us the
desired result.
\end{proof}

The chain rule bound be generalized to a $2$-dimensional
outermost function $f(x,y)$, as proved in Appendix 
\ref{sec:multidim_anal_proof}:

\begin{repcorollary}{cor:multivar_anal} Let $f(x,y)$ be analytic, with $\tilde{f}(x,y)$ be the function obtained by taking the
multivariate power series of $f$ and replacing all coefficients with their absolute values. Then, if
$g(\x)$ and $h(\x)$ are both efficiently learnable, $f(g(\x),h(\x))$ is as well with bound
\begin{equation}
\sqrt{M_{f\circ (g,h)}} = \left.\frac{d}{dy}\tilde{f}(\tilde{g}(y),\tilde{h}(y))\right|_{y=1}
\end{equation}
provided $\tilde{f}(\tilde{g}(y),\tilde{h}(y))$ converges at $y=1$.
\end{repcorollary}

\subsection{Learning dynamical systems}

\label{sec:learning_dynamical}

We can use the product and chain rules to show that many functions
important in scientific applications can be efficiently learnable.
This is true even when the function has a singularity. As an example
demonstrating both, we prove the following bound on learning Newton's law
of gravitation:

\begin{theorem}
\label{thm:gravity_bound}
Consider a system of $k$ bodies with positions $\x_{i}\in \mathbb{R}^{3}$
and masses $m_{i}$, interacting via the force:
\begin{equation}
\F_{i} = \sum_{j\neq i} \frac{m_{i}m_{j}}{r_{ij}^3}(\x_{j}-\x_{i})
\end{equation}
where $r_{ij} \equiv ||\x_{i}-\x_{j}||$.
We assume that $\rrat = r_{max}/r_{min}$, the ratio between the largest
and smallest pairwise distance between any two bodies, is constant. 
Suppose the $m_i$ have been rescaled
to be between $0$ and $1$.
Then the force law is efficiently
learnable in the sense of Definition \ref{def:eff_learnable} using the modified ReLU kernel
to generalization error less than $\eps$
using $k^{O(\ln(k/\eps))}$ samples.
\end{theorem}

\begin{proof}
We will prove learning bounds for each component of $F$ separately,
showing efficient learning with probability greater than $1-\delta/3k$. Then, using the union bound,
the probability of simultaneously learning all the components efficiently will be $1-\delta$.

There are two levels of approximation: first, we will construct a function which is within $\epsilon/2$
of the original force law, but more learnable. Secondly, we will prove bounds on learning that function
to within error $\epsilon/2$.

We first rescale the vector of collective $\{\x_{i}\}$ so that their collective length is at most $1$.
In these new units, this gives us $r_{max}^2\leq \frac{2}{k}$.
The first component
of the force on $\x_{1}$ can be written as:
\begin{equation}
(\F_{1})_{1} = \sum_{j=2}^{k} \frac{m_{1}m_{j}}{r_{1j}^{2}}\frac{((\x_{j})_{1}-(\x_{1})_{1})}{r_{1j}}
\end{equation} 
If we find a bound $\sqrt{M_{f}}$ for an individual contribution $f$
to the force, we can get a bound on the total
$\sqrt{M_{F}} = (k-1)\sqrt{M_{f}}$. Consider an individual force term
in the sum. The force has a singularity at
$r_{1j} = 0$. In addition, the function $r_{1j}$ itself is non-analytic due to the branch cut at $0$.

We instead will approximate the force law with a finite power series in $r_{1j}^2$, and get bounds on learning said
power series. The power series representation of
$(1-x)^{-3/2}$ is $\sum_{n=0}^{\infty} \frac{(2n+1)!!}{(2n)!!}x^{n}$.
If we approximate the function with $d$ terms, the error can be 
bounded using Taylor's theorem. The Lagrange form
of the error gives us the bound
\begin{equation}
\left|\frac{1}{(1-x)^{3/2}}-\sum_{n=0}^{d} \frac{(2n+1)!!}{ (2n)!!}x^{n}\right| \leq \frac{\sqrt{\pi d} |x|^{d+1}}{(1-|x|)^{5/2+d}}
\end{equation}
where we use $\frac{(2n+1)!!}{(2n)!!}\approx \sqrt{\pi n}$
for large $n$.
We can use the above expansion by
rewriting
\begin{equation}
r_{1j}^{-3} = a^{-3}(1-(1-r_{1j}^{2}/a^{2}))^{-3/2}
\end{equation}
for some shift $a$.
Approximation with $f_{d}(r_{1j}^2)$, the first $d$ terms of the power series in $(1-r_{1j}^{2}/a^{2})$ gives us the error:
\begin{equation}
|f_{d}(r_{1j}^2)-r_{1j}^{-3}|\leq \frac{\sqrt{\pi d}|1-r_{1j}^{2}/a^{2}|^{d+1}}{a^3(1-|1-r_{1j}^{2}/a^{2}|)^{5/2+d}}
\end{equation}
which we want to be small over the range $r_{min}\leq r_{1j}\leq r_{max}$.

The bound is optimized when it takes the same value at
$r_{min}$ and $r_{max}$, so we set
$a^{2} = (r_{min}^2+r_{max}^2)/2$.
In the limit that $r_{max}\gg r_{min}$, where learning is most difficult,
the bound becomes
\begin{equation}
|f_{d}(r_{1j}^2)-r_{1j}^{-3}|\leq 
\frac{\sqrt{8\pi d}}{r_{max}^3}
\left(\rrat^2/2\right)^{5/2+d}e^{-2(d+1)/\rrat^2}
\end{equation}
where $\rrat = r_{max}/r_{min}$, which is constant
by assumption.

In order to estimate an individual contribution to the force force to error $\epsilon/2k$ (so the total error
is $\epsilon/2$), we must have:
\begin{equation}
m_1 m_j r_{max}  |f_{d}(r_{1j})-r_{1j}^{-3}|\leq \frac{\eps}{2k}
\end{equation}
This allows us to choose the smallest $d$ which gives us this error. Taking
the logarithm of both sides, we have:
\begin{equation}
\frac{1}{2}\ln(d)-(5/2+d)\ln\left(2/\rrat^2\right)-2(d+1)/\rrat^2\leq \ln(\eps/k^2)
\end{equation}
where we use that
$r_{max}^2\leq 2/k$ after rescaling.
The choice $d\geq \rrat^2\ln(k^2/\eps)$ ensures error less than $\eps/2k$ per term.

Using this approximation, we can use the product and chain rules to 
get learning bounds on
the force law. We can write the approximation
\begin{equation}
F_{\eps}(\x) = \sum_{j\neq 1} m_1m_j f_{d}(h_{j}(\x)) k_{j}(\x)
\end{equation}
where $h_{j}(\x) = ||\x_{1}-\x_{j}||$
and $k_{j}(\x) = (\x_{1})_1-(\x_{j})_{j}$
The number of samples needed for efficient learning is bounded by
$\sqrt{M_{F_{\eps}}} = \frac{\sqrt{8}k}{r_{max}^3} A_{F_{\eps}}$, for
\begin{equation}
A_{F_{\eps}} =
\tilde{f}_{d}'(\tilde{h}(1))\tilde{h}'(1) \tilde{k}(1) +\tilde{f}_{d}(\tilde{h}(1)) \tilde{k}'(1)
\end{equation}
with
\begin{equation}
\tilde{k}(y) = \sqrt{2}y,~\tilde{h}(y) = 6y^2,~\tilde{f}_{d}(y) = \sqrt{\pi d}(1+y/a^2)^{d}
\end{equation}
Evaluating, we have
\begin{equation}
A_{F_{\eps}}   = 
 \sqrt{2\pi d}\left(1+\frac{12}{r_{max}^2}\right)^{d} +\sqrt{\pi d^3}\left(1+\frac{12}{r_{max}^2}\right)^{d-1}
\end{equation}
which, after using $r_{max}^2\leq 2/k$ and
$d = \rrat^2\ln(k^2/\eps)$
gives us the bound
\begin{equation}
\sqrt{M_{F_{\eps}}} \leq  k^{-1/2}  \left(\rrat^2\ln(k^2/\eps)\right)^{3/2}\left(24 k\right)^{R^2\ln(k^2/\eps)}
\label{eq:gravity_bound}
\end{equation}
The asymptotic behavior is
\begin{equation}
\sqrt{M_{F_{\eps}}} = k^{O(\ln(k/\eps))}
\end{equation}
since $\rrat$ is bounded.

We can therefore
learn an $\epsilon/2$-approximation of one component of $\F_{1}$, with probability
at least $1-\delta/3k$ and error $\epsilon/2$ with  $O(4(M_{F_{\eps}}+\log(3k/\delta))/\epsilon^2)$
samples. Therefore, we can learn $\F_{1}$ to error $\epsilon$ with the same number
of samples. Using a union bound, with probability at least $1-\delta$ we can simultaneously
learn all components of all $\{\F_{i}\}$ with that number of samples.
\end{proof}

We note that since the cutoff of the power series at $d(\eps) = O(\rrat^2\ln(k^2/\eps))$ dominates
the bound, we can easily compute learning bounds for other power-series kernels as well.
If the $d$th power series coefficient of the kernel is $b_d$, then 
the bound on $\sqrt{M_{F_{\eps}}}$
is increased by $(d(\eps)^2b_{d(\eps)})^{-1/2}$.
For example, for the Gaussian kernel, since $b_d^{-1/2} = \sqrt{d!}$, the bound becomes
\begin{equation}
\sqrt{M_{F_{\eps}}} = (\rrat^2\ln(k^2/\eps)k)^{O(\ln(k/\eps))}
\end{equation}
which increases the exponent of $k$ by a
factor of $\ln(\rrat^2\ln(k^2/\eps))$.

\section{Experiments}

\label{sec:expts}

We empirically validated our analytical learning bounds by
training models to
learn the gravitational force function for $k$ bodies (with $k$ 
ranging from $5$ to $400$) in a $3-$dimensional space.
We created synthetic datasets by randomly drawing $k$ points from
$[0, 1]^3$ corresponding to the location of $k$ bodies, and compute 
the gravitational force (according to Equation \ref{eq:gravity}) on a
target body also drawn randomly from $[0, 1]^3$. To avoid 
singularities, we ensured a minimum distance of $0.1$ between the 
target body and the other bodies
(corresponding to the choice $\rrat = 10$). 
As predicted by the theory, none of the models learn well
if $\rrat$ is not fixed.
We randomly drew the 
masses 
corresponding to the $k+1$ bodies from $[0,10]$. We generated $5$ 
million such examples - each example with $4(k+1)$ features 
corresponding to the location and mass of each of the bodies, and a 
single label corresponding to the gravitational force $F$ on the target 
body along the $x$-axis. We held out 
$10\%$ of the dataset as test data to compute the root mean square 
error (RMSE) in prediction.
We trained three different neural networks 
on this data, corresponding to various kernels we analyzed in the 
previous section:
\begin{enumerate}
\item A wide one hidden-layer ReLU network (corresponding
to the ReLU NTK kernel).
\item A wide one hidden-layer ReLU network with a 
constant bias feature added to the input (corresponding to the NTK 
kernel in Equation \ref{eq:ntk-bias}).
\item A wide one hidden-layer network with exponential 
activation function, where only the top layer of the network is
trained (corresponding to the Gaussian kernel).
\end{enumerate}

We used a hidden layer of width $1000$ for all the
networks, as we observed that 
increasing the network width further did not improve results 
significantly.
All the hidden layer weights were initialized randomly.

In Figure \ref{fig:expts} we show the the normalized RMSE
(RMSE/[$F_{max}-F_{min}$]) for each of the neural networks for 
different values of the number of bodies $k$.


\begin{figure}[h!]
\centering
\includegraphics[width=\linewidth]{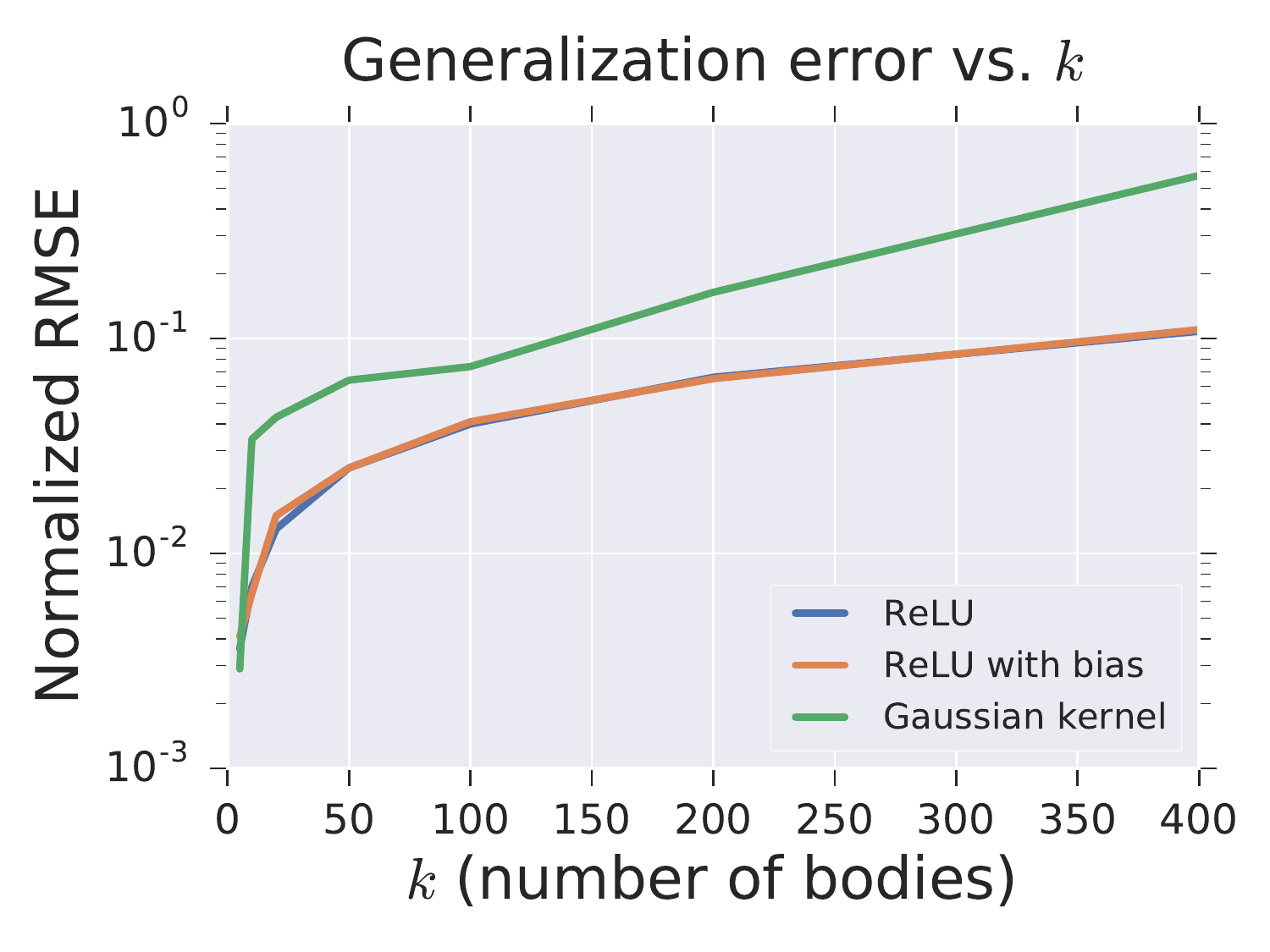}
\caption{RMSE vs number of bodies $k$ for learning gravitational
force law for different kernels. Normalized by the range $F_{max}-F_{min}$
of the forces. Gaussian kernels learn worse than
ReLU at large $k$.
}
\label{fig:expts}
\end{figure}

All three networks are able to learn the 
gravitational force equation with small normalized RMSE for hundreds of bodies.  
Both the ReLU network and ReLU with bias outperform the network corresponding to the Gaussian
kernel (in terms of RMSE) as $k$ increases.
In particular, the Gaussian kernel learning seems to quickly degrade at around $400$ bodies,
with a normalized RMSE exceeding $50\%$. This is consistent with the learning bounds
for these kernels in Section \ref{sec:kernel_learning}, and suggests
that those bounds may in fact be useful to compare the performances
of different networks in practice.

We did not, however, observe much difference in the performance of the
ReLU network when adding a bias  to the input, which suggests that 
the inability to get an analytical bound due
to only even powers
in the ReLU NTK kernel might be a shortcoming of the proof
technique, rather than a property which fundamentally limits the
model.

\section{Conclusions}

Our theoretical work shows that the broad and important class of 
analytic functions is provably
learnable with the right kernels (or the equivalent wide networks). 
The methods which we developed may
be useful for proving learnability of other classes of functions, 
such as the flows induced by
finite-time integration of differential equations. Furthermore, in general, there is an open question as to whether these generalization bounds can be substantially improved.

Our experiments suggest that these
bounds may be useful for distinguishing which types of models are 
suited for specific problems.
Further experimental and theoretical work is necessary to ascertain whether this holds for the finite-width
networks used in practice, or when common hyperparameter 
tuning/regularization strategies are used
during training, such as ARD in kernel learning.

\bibliography{poly_learning_refs_aa}
\bibliographystyle{icml2020}

\clearpage

\appendix

\section{Kernels and two layer networks}

\label{sec:kernel_two_layer}

Previous work focused on generalization bounds for training the hidden layers of wide networks with SGD.
Here we show that these bounds also apply to the case where only the final layer weights
are trained (corresponding to the NNGP kernel in \cite{lee_wide_2019}). The proof strategy
consists of showing that finite-width networks have a sensible infinite-width limit, and showing that
training causes only a small change in parameters of the network.

Let $m$ be the number of hidden units, and $n$ be the number of data points.
Let $\h$ be a $n \times m$ random matrix denoting the activations of the hidden layer (as a function of the weights of the 
lower layer) for all $n$ data points. 
Similarly to
\cite{arora_finegrained_2019, du_gradient_2019} we will argue that for large enough $m$ even if we take a random input 
layer and just train the upper layer weights then the generalization error is at most 
$\sqrt{\frac{\y^{\tpose}(\H)^{-1}\y}{n}}$. For our purposes, we define:
\begin{equation}
 \H = \expect[\h\h^{\tpose}]   
\end{equation}
which is the NNGP kernel from \cite{lee_wide_2019}.

If $K(\x,\x')$, the kernel function which generates $\H$
is given by a infinite Taylor series in $\x\cdot\x'$ it can be argued that 
$\H$ has full rank for most real world distributions. For example, the ReLU activation this holds as long as no two data 
points are co-linear (see Definition 5.1 in \cite{arora_finegrained_2019}). We can prove this more explicitly in the general case.

\begin{lemma}
If all the $n$ data points $x$ are distinct and the Taylor series of  $K(\x,\x')$ in $\x\cdot\x'$ has positive coefficients everywhere then $\H$ is not singular.
\end{lemma}
\begin{proof}
First consider the case where the input $x$ is a scalar.
Since the Taylor series 
corresponding to $K(x,x')$ consists of monomials of all degrees of
$xx'$, we can view it as some inner product in a kernel space induced by
the function $\Phi(x) = (1,x,x^2,\ldots)$, where the inner product is diagonal
(but with potentially different weights) in this basis.
For any distinct set of inputs $\{x_1, .., x_n\}$ the set of vectors $\Phi(x_i)$ are linearly independent.
The first $n$ columns produce the Vandermonde matrix obtained by stacking rows
$1,x,x,...,x^{n-1}$ for $n$ different values of $x$, which is well known to be non-singular
(since a zero eigenvector would correspond to a degree $n-1$ polynomial with $n$ distinct roots
$\{x_1, .., x_n\}$).

This extends to the case of multidimensional $\x$ if the values, projected
along some dimension, are distinct.
In this case, the kernel space corresponds to the direct sum of copies of $\Phi$ applied
elementwise to each coordinate $\x_{i}$. If all the points are distinct and
and far apart from each other, the probability that a given pair coincides under random projection
is negligible. From a union bound, the probability that a given pair coincide is also bounded --
so there must be directions such that projections along that direction are distinct.
Therefore, $\H$ can be considered to be invertible in general.
\end{proof}

As $m \rightarrow \infty$, $\h\h^{\tpose}$ 
concentrates to its expected value. More precisely,
$(\h\h^{\tpose})^{-1}$
approaches $(\H)^{-1}$ for large $m$ if we assume that the smallest 
eigenvalue $\lambda_{min}(\H) \geq \lambda_0$, which
from the above lemma we know
to be true for fixed $n$.
(For the ReLU NTK kernel the difference becomes negligible with high probability for
$m = poly(n/\lambda_0)$ \cite{arora_finegrained_2019}.)
This allows us to replace $\h\h^{\tpose}$ with $\H$ in any bounds involving the former.

We can get learning bounds in terms of $\h\h^{\tpose}$ in the following manner.
The output of the network is given by $\y = \w\cdot\h$, where $\w$ is the vector of upper layer weights
and $\y$ is $1\times n$ vector of
training output values.
The outputs are linear in $\w$.
Training only the $\w$, and assuming $\h\h^{\tpose}$ is invertible
(which the above arguments show is true with high probability
for large $m$), the following lemma holds:
\begin{lemma} If we initialize a random lower layer and train the weights of the upper layer,
then there exists a solution $\w$
with norm $\y^{\tpose} (\h\h^{\tpose})^{-1} \y$.
\end{lemma}

\begin{proof}
The minimum norm solution to $\y = \w^{\tpose}\h$ is
\begin{equation}
\w^* = (\h^{\tpose}\h)^{-1}\h^{\tpose}\y
\end{equation}
The norm $(\w^*)^{\tpose}\w^{*}$ of this solution is
given by $\y^{\tpose}\h(\h^{\tpose}\h)^{-2}\h^{\tpose}\y$.

We claim that $\h(\h^{\tpose}\h)^{-2}\h^{\tpose} = (\h\h^{\tpose})^{-1}$.
To show this, consider the SVD decomposition $\h = \m{U}\m{S}\m{V}^{\tpose}$.
Expanding we have
\begin{equation}
\h(\h^{\tpose}\h)^{-2}\h^{\tpose} = \m{U}\m{S}\m{V}^{\tpose} (\m{V}\m{S}^{2} \m{V}^{\tpose})^{-2} \m{V}\m{S}\m{U}^{\tpose}
\end{equation}
Evaluating the right hand side gets us $\m{U}\m{S}^{-2}\m{U}^{\tpose} = (\h\h^{\tpose})^{-1}$.

Therefore, the norm of the minimum norm solution is $\y^{\tpose} (\h\h^{\tpose})^{-1}\y$.
\end{proof}

For large $m$, the norm approaches $\y^{\tpose} (\H)^{-1}\y$.
Since the lower layer is fixed, the optimization problem is linear and therefore convex in the trained weights
$\w$. Therefore
SGD with small learning rate will reach
this optimal solution. The 
Rademacher complexity of this function class is at most $\sqrt{\frac{\y^{\tpose}(\H)^{-1}\y}{2n}}$.
The optimal solution has $0$ train error based on the 
assumption that $\H$ is full rank and the test error will be no more than this Rademacher complexity - 
identical to the previous results for training a ReLu network \cite{arora_finegrained_2019, du_gradient_2019}.

Note that 
although we have argued here assuming only upper layer is trained,
\cite{arora_finegrained_2019, du_gradient_2019, andoni_learning_2014} show that even if both layers are trained for large 
enough $m$ the training dynamics is governed by the NTK kernel,
and the lower layer changes so little over the training steps that $\h \h^{\tpose}$ remains close to 
$\H$ through the gradient descent.

\section{Learning polynomials with kernels}

\subsection{Modified ReLU power series}

\label{sec:mod_ReLU}

One limitation of the bounds in \cite{arora_finegrained_2019} is that they were limited
to even-degree polynomials only. One way to overcome that limitation is to compute the NTK
kernel of a modified ReLU network. Given an input $\x\in\mathbb{R}^{d}$,
the actual input to the vector is the $d+1$ dimensional vector $(\x/\sqrt{2},1/\sqrt{1})$.
The resulting kernel is
then given by:
\begin{equation}
K(\x,\x') = \frac{\x\cdot\x'+1}{4\pi}\left(\pi-\arccos\left(\frac{\x\cdot\x'+1}{2}\right)\right)
\end{equation}
for inputs $\x$ and $\y$.
We can use the power series of $\arccos$ around $0$ to compute the power series around
$\frac{1}{2}$:
\begin{equation}
\begin{split}
\pi-&\arccos\left(\frac{x+1}{2}\right)  = \\
& \frac{\pi}{2}+ \sum_{n=1}^{\infty} \frac{1}{(2n+1)}\frac{(2n-1)!!}{(2n)!!}\left(\frac{1+x}{2}\right)^{2n+1}
\end{split}
\end{equation}
Evaluating the monomial, we have the double sum
\begin{equation}
\begin{split}
\pi-&\arccos\left(\frac{x+1}{2}\right) = \\
& \frac{\pi}{2}+\sum_{n=1}^{\infty} \frac{1}{(2n+1)}\frac{(2n-1)!!}{(2n)!!} \sum_{k=0}^{2n+1}2^{-(2n+1)}\binom{2n+1}{k} x^k
\end{split}
\end{equation}
Let $\sum_{k} b_k x^k$ be the power series representation about $0$. It is difficult to represent
$b_k$ in terms of simple functions; nevertheless, we can asymptotically compute the form of $b_k$ for
large $k$. We have:
\begin{equation}
b_k = \sum_{n=1}^{\infty} \frac{1}{(2n+1)}\frac{(2n-1)!!}{(2n)!!} 2^{-(2n+1)}\binom{2n+1}{k}
\end{equation}
For large $n$ and $k$, using Stirling's approximation we can write:
\begin{equation}
b_k \approx \sum_{n=(k-1)/2}^{\infty}\frac{\sqrt{2}}{\pi} \frac{1}{(2n+1)^2} e^{-((2n+1)/2 - k)^2/(2n+1)}
\end{equation}
correct up to multiplicative error $o(1)$. The saddle point approximation gives us that the sum is dominated
by terms of order $n = k\pm\sqrt{k}$; therefore the approximate value of the sum is
\begin{equation}
b_k \approx \frac{1}{2\sqrt{\pi}} k^{-3/2}
\end{equation}
This has the same asymptotic form as the unmodified ReLU kernel without modification, but is now non-zero for
odd powers as well. Therefore, the learning bounds
from \cite{arora_finegrained_2019} (as well as the extension in Lemma \ref{lem:multivar}) will
hold for the modified ReLU kernel, for all orders of polynomials.

\subsection{Proof of Lemma \ref{lem:multivar}}

\label{sec:multivar_proof}

The following extension of Corollary 6.2 from \cite{arora_finegrained_2019} is vital
for proving learning bounds on multivariate functions:

\begin{replemma}{lem:multivar}
Given a collection of $p$ vectors $\bbet_{i}$ in $\mathbb{R}^d$,
the function $f(\x) = \prod_{i=1}^{p} \bbet_{i}\cdot \x$ is efficiently learnable
in the sense of Definition \ref{def:eff_learnable}
using the modified ReLU kernel with
\begin{equation}
\sqrt{M_{f}} = p\prod_{i=1}^{p}\bnorm_{i}
\end{equation}
where $\bnorm_{i}\equiv ||\bbet_{i}||_{2}$.
\end{replemma}

\begin{proof}
The proof of Corollary 6.2 in \cite{arora_finegrained_2019} relied on the following
statement: given positive semi-definite matrices $\m{A}$ and $\m{B}$, with $\m{A} \succeq \m{B}$,
we have:
\begin{equation}
\m{P}_{\m{B}}\m{A}^{-1}\m{P}_{\m{B}} \preceq \m{B}^{+}
\label{eq:psd_ineq}
\end{equation}
where $+$ is the Moore-Penrose pseudoinverse,
and $\m{P}$ is the projection operator.

We can use this result, along with the Taylor expansion of the kernel and a
particular
decomposition of a multivariate monomial in the following way.
Let the matrix $\X$ to be the training data,
such that the $\a$th column $\x_{i}$ is a unit vector in $\mathbb{R}^d$.
Given $\m{K}\equiv \m{X}^{\tpose}\m{X}$, the matrix
of inner products, the Gram matrix $\H$ of the kernel can be written as
\begin{equation}
\H = \sum_{k=0}^{\infty} b_k \m{K}^{\circ k}
\end{equation}
where $\circ$ is the Hadamard (elementwise) product.
Consider the problem of learning the function $f(\x) = \prod_{i=1}^{p} \bbet_{i}\cdot\x$.
Note that we can write:
\begin{equation}
f(\X) = (\X^{\odot k})^{\tpose} \otimes_{i=1}^{k} \bbet_{i}
\end{equation}

Here $\otimes$ is the tensor product, which for vectors takes an $n_1$-dimensional
vector and an $n_{2}$ dimensional vector as inputs vectors and
returns a $n_{1}n_{2}$ dimensional vector:
\begin{equation}
\m{w}\otimes\m{v} = \begin{pmatrix}
w_1v_1\\
w_{1}v_{2}\\
\cdots\\
w_{1}v_{n_{2}}\\
w_{2}v_{1}\\
\cdots\\
w_{n_{1}}v_{n_{2}}
\end{pmatrix}
\end{equation}
The operator $\odot$ is the Khatri-Rao product, which takes
an $n_{1}\times n_{3}$ matrix
$\m{A} = (\m{a}_{1},\cdots,\m{a}_{n_{3}})$ and a $n_{2}\otimes n_{3}$ matrix $\m{B} = (\m{b}_{1},\cdots,\m{b}_{n_{3}})$
and returns the $n_{1}n_{2}\times n_{3}$ dimensional matrix
\begin{equation}
\m{A}\odot\m{B} = (\m{a}_{1}\otimes\m{b}_{1},\cdots,\m{a}_{n_{3}}\otimes\m{b}_{n_{3}})
\end{equation}
For $p=2$, this form of $f(\X)$ can be proved explicitly:
\begin{equation}
(\X^{\odot 2})^{\tpose} \bbet_{1}\otimes\bbet_{2} =
\begin{pmatrix}
\x_{1}\otimes\x_{1},\cdots,\x_{P}\otimes\x_{P}
\end{pmatrix}^{\tpose}\bbet_{1}\otimes\bbet_{2}
\end{equation}
The $\a$th element of the matrix product is
\begin{equation}
(\x_{\a}\otimes\x_{\a})\cdot(\bbet_{1}\otimes\bbet_{2}) = (\bbet_{1}\cdot\x_{\a})(\bbet_{2}\cdot\x_{\a})
\end{equation}
which is exactly $f(\x_{\a})$.
The formula can be proved for $p>2$ by finite induction.

With this form of $f(\X)$, we can follow the steps of the proof in \cite{arora_finegrained_2019},
which was written for the case where the $\bbet_{i}$ were identical:
\begin{equation}
\y^{\tpose}(\m{H}^{\infty})^{-1}\y = (\otimes_{i=1}^{p} \bbet_{i})^{\tpose}\X^{\odot p}  (\m{H}^{\infty})^{-1} (\X^{\odot p})^{\tpose} \otimes_{i=1}^{p} \bbet_{i}
\end{equation}
Using Equation \ref{eq:psd_ineq}, applied to $\m{K}^{\circ p}$, we have:
\begin{equation}
\begin{split}
\y^{\tpose}&(\m{H}^{\infty})^{-1}\y \leq\\ 
 &b_p^{-1}(\otimes_{i=1}^{p} \bbet_{i})^{\tpose}\X^{\odot p} \m{P}_{\m{K}^{\circ p}} (\m{K}^{\circ p})^{+} \m{P}_{\m{K}^{\circ p}} (\X^{\odot p})^{\tpose} \otimes_{i=1}^{p} \bbet_{i}
 \end{split}
\end{equation}
Since the $\X^{\odot p}$ are eigenvectors of $\m{P}_{\m{K}^{\circ p}}$ with eigenvalue $1$, and
$\X^{\odot p}(\m{K}^{\circ p})^{+}(\X^{\odot p})^{\tpose} = \m{P}_{\X^{\odot p}}$, we have:
\begin{equation}
\y^{\tpose}(\m{H}^{\infty})^{-1}\y \leq 
 b_p^{-1}(\otimes_{i=1}^{p} \bbet_{i})^{\tpose}\m{P}_{\X^{\odot p}} \otimes_{i=1}^{p} \bbet_{i}
\end{equation}
\begin{equation}
\y^{\tpose}(\m{H}^{\infty})^{-1}\y \leq b_p^{-1} \prod_{i=1}^{p} \bbet_{i}\cdot\bbet_{i}
\end{equation}

For the modified ReLU kernel, $b_p \geq p^{-2}$. Therefore, we have
$\sqrt{\y^{\tpose}(\m{H}^{\infty})^{-1}\y}\leq \sqrt{M_{f}}$ for
\begin{equation}
\sqrt{M_{f}} = p\prod_{i}\bnorm_{i}
\end{equation}
where $\bnorm_{i}\equiv ||\bbet_{i}||_{2}$, as desired.
\end{proof}

\section{Multi-dimensional analytic functions}

\label{sec:multidim_anal_proof}

In the main text we proved results for analytic functions comprised
of composing a univariate analytic function $g(y)$ with a simple
multivariate function $h(\x)$. However, we can prove bounds
for functions comprised of a composition of a multivariate
analytic function $f(x,y)$ with efficiently learnable
analytic functions $g(\x)$ and $h(\x)$. More formally,
we prove the following:

\setcounter{theorem}{1}
\setcounter{corollary}{2}

\begin{corollary}
\label{cor:multivar_anal}
Let $f(x,y)$ be analytic, with $\tilde{f}(x,y)$ be the function obtained by taking the
multivariate power series of $f$ and replacing all coefficients with their absolute values. Then, if
$g(\x)$ and $h(\x)$ are both efficiently learnable, $f(g(\x),h(\x))$ is as well with bound
\begin{equation}
\sqrt{M_{f\circ (g,h)}} = \left.\frac{d}{dy}\tilde{f}(\tilde{g}(y),\tilde{h}(y))\right|_{y=1}
\end{equation}
provided $\tilde{f}(\tilde{g}(y),\tilde{h}(y))$ converges at $y=1$.
\end{corollary}

\begin{proof}
Expand the power series of representation of $f(g(\x),h(\x))$. Replacing all terms with their
absolute values, multiplied by $y^k$ for the appropriate power $k$, we get the power series
$\tilde{f}(\tilde{g}(y),\tilde{h}(y))$. Replacing the terms in the power series of $f(g(\x),h(\x))$ with their individual bounds,
and comparing with the power series representation of $\frac{d}{dy}\tilde{f}(\tilde{g}(y),\tilde{h}(y))$, we arrive at the
result.
\end{proof}

Generalizations to higher dimensional $f$ proceed similarly.

\end{document}